\documentclass[authoryear,preprint,11pt]{elsarticle}
\usepackage{enumitem}
\usepackage{amssymb}
\usepackage{eurosym}
\usepackage{color}
\usepackage{bm}
\usepackage{booktabs}
\usepackage{enumitem}
\usepackage{multirow}
\usepackage{longtable}
\usepackage{arydshln}
\usepackage[top=2.2cm,bottom=2.5cm,left=2.5cm,right=2.5cm]{geometry}
\usepackage{hyperref}
\usepackage{mathtools}
\usepackage{algorithm,algpseudocode}
\usepackage{subcaption}
\usepackage{tikz}
\usepackage{pgfplots}
\usetikzlibrary{patterns}
\usetikzlibrary{arrows}
\usepackage{amsmath}

\usepackage{amsthm,amsmath,eurosym}

\newtheorem{definition}{Definition}
\newtheorem{theorem}{Theorem}
\newtheorem{proposition}{Proposition}

\newtheorem{cor}{Corollary}
\newcommand{\R}{\mathbb{R}}

\newcommand{\pt}{\text{ }\forall\text{ }}
\newcommand{\tq}{\text{ }:\text{ }}
\newcommand{\N}{\mathbb{N}}

\newcommand{\io}{[0,1]}
\newcommand{\iol}{]0,1]}

\newcommand{\cf}{\mathcal{F}}
\newcommand{\cd}{\mathcal{D}}
\newcommand{\cl}{\mathcal{L}}
\newcommand{\ci}{\mathcal{I}}
\newcommand{\cff}{{\mathcal{F}_2}}

\pgfplotsset{compat=1.18} 

\journal{European Journal of Operational Research}

\usepackage{lineno, blindtext}

\makeatletter
\def\ps@pprintTitle{%
  \let\@oddhead\@empty
  \let\@evenhead\@empty
  \def\@oddfoot{\reset@font\hfil\thepage\hfil}
  \let\@evenfoot\@oddfoot
}
\makeatother

\begin{document}

\begin{frontmatter}
    \title{Building Interval Type-2 Fuzzy Membership Function: A Deck of Cards based Co-constructive Approach }
    \author[JAEN]{Bapi {\sc Dutta\corref{cor2}}}\ead{bdutta@ujaen.es}
    \author[JAEN1]{Diego {\sc  Garc{\'i}a-Zamora}}\ead{dgzamora@ujaen.es}
    
    \author[CEGIST]{Jos{\'e} {\sc Rui~Figueira}}\ead{figueira@tecnico.ulisboa.pt }
    \author[JAEN]{Luis {\sc Mart{\'i}nez}}\ead{martin@ujaen.es}
    \address[JAEN]{Department of Computer Science, University of Ja{\'e}n, 23071 Ja{\'e}n, Spain}
     \address[JAEN1]{Department of Mathematics, University
of Ja{\'e}n, 23071 Ja{\'e}n, Spain}
        \address[CEGIST]{CEGIST, Instituto Superior T\'{e}cnico,  Universidade de Lisboa, Portugal}

    \begin{abstract}
    \noindent Since its inception, Fuzzy Set has been widely used to handle uncertainty and imprecision in decision-making. However, conventional fuzzy sets, often referred to as type-1 fuzzy sets (T1FSs) have limitations in capturing higher levels of uncertainty, particularly when decision-makers (DMs) express hesitation or ambiguity in membership degree. To address this, Interval Type-2 Fuzzy Sets (IT2FSs) have been introduced by incorporating uncertainty in membership degree allocation, which enhanced flexibility in modelling subjective judgments. Despite their advantages, existing IT2FS construction methods often lack active involvement from DMs and that limits the interpretability and effectiveness of decision models. This study proposes a socio-technical co-constructive approach for developing IT2FS models of linguistic terms by facilitating the active involvement of DMs in preference elicitation and its application in multicriteria decision-making (MCDM) problems. Our methodology is structured in two phases. The first phase involves an interactive process between the DM and the decision analyst, in which a modified version of Deck-of-Cards (DoC) method is proposed to construct T1FS membership functions on a ratio scale. We then extend this method to incorporate ambiguity in subjective judgment and that resulted in an IT2FS model that better captures uncertainty in DM’s linguistic assessments. The second phase formalizes the constructed IT2FS model for application in MCDM by defining an appropriate mathematical representation of such information, aggregation rules, and an admissible ordering principle. The proposed framework enhances the reliability and effectiveness of fuzzy decision-making not only by accurately representing DM’s personalized semantics of linguistic information.
    \end{abstract}
    \vspace{0.25cm}
\end{frontmatter}

\section{Introduction}\label{sec: Introduction}
\noindent  Fuzzy Set Theory, introduced by Zadeh in 1965, provides a powerful mathematical framework for handling uncertainty and imprecision in decision-making associated with human subjective judgments\citep{zadeh1965, bellman1970}. Unlike classical sets, where an element either belongs or does not belong to a set, fuzzy sets allow partial belongingness through membership degree, represented by a membership function by assigning every element to a membership degree from the unit interval $[0,1]$ \citep{Klir1994}. This modelling flexibility has made fuzzy sets widely applicable in decision-making scenarios involving subjective judgments, linguistic assessments, and vague information including multi-criteria decision-making (MCDM) \citep{pedrycz2011fuzzy}.  

However, conventional fuzzy sets, often referred to as Type-1 Fuzzy Sets (T1FS), are unable to model higher levels of uncertainty, especially when decision-makers express their opinions with hesitation, ambiguity, or vague information \citep{mendel2002type, torra2010hesitant}. One key limitation is that the membership degree itself is precise, meaning it cannot fully capture the uncertainty inherent in human cognition and linguistic assessments \citep{mendel2007type}. To overcome these drawbacks, Interval Type-2 Fuzzy Sets (IT2FSs) were introduced as an extension of Type-1 Fuzzy Sets (T1FSs), in which, the membership degree itself is fuzzy, meaning it is defined by an interval ($\subset [0,1]$) rather than a precise value \citep{zadeh1975concept,liang2000interval, bustince2015historical}. This enhanced ability to capture uncertainty in human judgments enables the development of a more robust and reliable decision-making framework based on IT2FS, effectively integrating expert opinions and subjective assessments \citep{celik2015,chen2010fuzzy}. However, constructing such IT2FSs for decision-making is not straightforward, as there are many potential challenges, including capturing the subjectivity of human judgments, which could impact the effectiveness and efficiency of the decision-making process.

In expert-driven decision-making, IT2FSs are used to model the fuzzy meaning of linguistic terms associated with words \citep{kahraman2014fuzzy}. There are several ways to construct such IT2FS model for linguistic term in the literature. One common approach frequently observes in the literature is that it is directly provided by an expert by eliciting the parameters/shape associated with the upper and lower membership functions \citep{chen2014electre}, even often directly assume encoding of the linguistic terms by predefined IT2FSs \citep{hu2013multi}. There is a small literature about obtaining the IT2FS model for a linguistic term by collecting interval data on the linguistic term from a group of subjects or a single subject through survey and processing them to elicit the parameters of upper and lower memberships such as Interval approach \citep{liu2008encoding}, Extended Interval Approach \citep{wu2011enhanced} and HM Approach \citep{hao2015encoding}. Construction from T1FS(s) given by the experts or a single expert is also proposed in the literature \citep{pagola2012interval}. Although capturing uncertainty in human subjectivity is essential for building such model for linguistic term, decision-makers are rarely involved in a close interaction with a decision analyst during the construction process. This lack of interaction undermines the proper elicitation of uncertainty associated with subjective judgments and affects the interpretability of the model to the decision-maker.

The construction of reliable IT2FS model for linguistic term is the key to successfully applying them to expert-driven decision-making process \citep{muhuri2017user}. However, most of the existing applications of these methods in solving decision-making problem do not involve expert(s) fully in the construction IT2FS model \citep{gupta2022gentle,alaei2025efficient,wan2021integrated}. Instead, it has been assumed that a predefined IT2FFs model for linguistic terms and DMs are asked to provide the linguistic terms only \citep{celik2015}. Further, the studies that directly ask to provide the parameters of the upper and lower memberships require high cognitive effort and knowledge of the IT2FS \citep{chen2015interval, chen2014electre}, which makes it very difficult for DMs to express their subjective judgments appropriately. Additionally, in the survey-based approaches the DMs involvement is limited to provide an interval estimate regarding the linguistic terms \citep{mendel2010perceptual}, and it is hard to interpret how IT2FS model of the linguistic term is obtained from that interval without their involvement in the process along with the elicitation uncertainty \citep{baratimehr2023}. To overcome these issues, a socio-technical co-construction approach is required, ensuring not only the active involvement of decision-makers in a dynamic and participatory process but also the effective elicitation of their subjective judgments \citep{bottero2018, costa2022}. In a nutshell, there is scope to develop methodologies that more effectively integrate the evolving preferences and insights of decision-makers in constructing IT2FS models for linguistic terms.

In this study, we propose a socio-technical and co-constructive approach to facilitate the construction of the expert-driven IT2FS model of the linguistic terms and develop a framework for solving fuzzy decision-making problems within the MCDM framework. The proposed methodologies could be divided into two phases. In the first phase, we utilize the co-constructive approach between decision-maker and decision analyst for the interactive preference elicitation process in building IT2FSs model for linguistic terms associated with a linguistic scale. The basis of the construction process is a modified version of Deck-of-Card (DoC) method \citep{corrente2021}, which enables us to build the membership function of T1FS in a ratio scale that allows meaningful interpretations of the relationship between membership degrees of any two elements \citep{Marchant2003}. Employing this process, we develop a three-step methodology to build an IT2FS model for a linguistic scale. In the first step, we map the labels of the linguistic terms onto a value scale \citep{Figueira2010,bottero2015,figueira2023}. Subsequently, in step two, the core and support of IT2FS model for each linguistic term are identified through a co-constructive process proposed by \cite{DoC-MF}. In the final step, we incorporate hesitation of DM's subjective judgments into the modelling of modified DoC method for constructing T1MF, and that subsequently induced uncertainty into membership degree allocations and resulted in upper and lower membership function of IT2FS model of the linguistic terms. 

In the second phase, we focus on mathematically formalize IT2FS model of the linguistic terms derived through the proposed socio-technical and co-constructive approach to enable its application in the expert-driven MCDM problem. For that purpose, we first define their arithmetic operational laws for such information based on which the aggregation rule is developed. We further define the total ordering of this kind of IT2FSs through the admissible ordering principle of the T1MF \citep{GARCIAZAMORA2024108863}. Finally, this information fusion mechanisms are utilized in the resolution of a practical problem in MCDM context. 

The paper is organized as follows. In Section \ref{sec: prelim}, the basic concepts of fuzzy sets and their operational laws, and concepts of type-2 fuzzy sets are introduced. It also provides a brief primer on the construction of DoC-T1MF, its mathematical representation and operational rules. Section \ref{sec: DoCMembership} propose the construction method of DoC-T2MF through a socio-technical co-constructive process by utilizing a modified version of DoC that can build membership function on a ratio scale. In Section \ref{sec: Aggregation Functions}, we formalize the mathematical representation of DoC-T2MF and introduce the aggregation and ordering rules for such information. Finally, we make concluding remarks of our study in Section \ref{sec: Conclusions}.




    

\section{Preliminaries}\label{sec: prelim}
\noindent  This section aims to set the definitions and notations we use throughout the manuscript.
\subsection{Some notions on Fuzzy Sets Theory}\label{subsec: FuzzySets}
\noindent We begin by introducing Fuzzy Sets Theory \citep{zadeh1965}. A fuzzy set, defined over a given universe of discourse, represented by $X$, generalizes the concept of a set's membership function. It is worth noting that any subset $A \subseteq X$ can be described by its characteristic function $\chi_A: X \to \{0,1\}$, which is defined as follows:

\begin{equation*}
\chi_A(x)=\left\{
\begin{array}{ll}
1 & \text{ if } x \in A,\\
0 & \text{ if } x \notin A.
\end{array}
\right.
\end{equation*}

In contrast, fuzzy sets allow membership functions to take values between 0 or 1.

\begin{definition}[Fuzzy Sets \citep{zadeh1965}]
Let $X$ represent the universe of discourse. A function $A: X \to [0,1]$ is called a fuzzy set on $X$. The value $A(x)$ at a given $x \in X$ signifies the degree to which $x$ belongs to $A$. The set of all fuzzy sets on $X$ is denoted as $\mathcal{F}(X)$. The support of $A$ is the set $\text{supp}_A = \{x \in X : A(x) > 0\}$. The core of $A$ is the set $\text{core}_A = \{x \in X : A(x) = 1\}$. A fuzzy set is called normal if its core is non-empty. For any $\alpha \in (0,1]$, the $\alpha$-cut of $A$ is defined as $A_\alpha = \{x \in X : A(x) \geqslant \alpha\}$. Additionally, $A_0$ refers to the topological closure of $\text{supp}(A)$.
\end{definition}

Fuzzy numbers represent a specific class of fuzzy sets that extend the notion of numbers on the real line.

\begin{definition}[Fuzzy number \citep{Klir1994}]
   \noindent A fuzzy set $A: \mathbb{R} \to [0,1]$ on the real numbers $\mathbb{R}$ is called a fuzzy number if $A$ is normal, and for each $\alpha \in [0,1]$, the $\alpha$-cut $A_\alpha$ is a bounded and closed interval.
    For any $\alpha \in [0,1]$, we represent $A_\alpha$ as the interval $[A_\alpha^-, A_\alpha^+]$, where $A_\alpha^-$ and $A_\alpha^+$ are the lower and upper bounds of the $\alpha$-cut, respectively. The set of all fuzzy numbers whose support is contained in $\io$ is denoted by $\cf$. 
\end{definition}

The standard arithmetic operations on real numbers can be extended to fuzzy numbers as follows \citep{Klir1994}.

\begin{definition}[Fuzzy arithmetic]
\label{def:aggregation_fn}
Let $A,B \in \mathcal{F}(\mathbb{R})$. Define $A \oplus B$ and $A \odot B$ in $\mathcal{F}(\mathbb{R})$ as
\begin{equation*}
(A \oplus B)(z) = \operatorname*{sup}_{x+y=z} \operatorname*{min}\{A(x), B(y)\}, \;
(A \odot B)(z) = \operatorname*{sup}_{x \cdot y=z} \operatorname*{min}\{A(x), B(y)\}
\end{equation*}
for all $z \in \mathbb{R}$.
\end{definition}

Moreover, these operations on fuzzy numbers can be expressed through $\alpha$-cuts \citep{Klir1994}.

\begin{proposition}[$\alpha$-cut representation for fuzzy arithmetic]\label{prop:sum-prod-fn}
    Let $A, B \in \mathcal{F}(\mathbb{R})$ and $\alpha \in (0,1]$. Then:
    \begin{align*}
        (A \oplus B)_\alpha &= [A_\alpha^- + B_\alpha^-, A_\alpha^+ + B_\alpha^+] \\
        (A \odot B)_\alpha &= [\min\{A_\alpha^-B_\alpha^-, A_\alpha^+B_\alpha^-, A_\alpha^-B_\alpha^+, A_\alpha^+B_\alpha^+\}, \\
        &\max\{A_\alpha^-B_\alpha^-, A_\alpha^+B_\alpha^-, A_\alpha^-B_\alpha^+, A_\alpha^+B_\alpha^+\}].
    \end{align*}
\end{proposition}

When fuzzy numbers are used in decision-making, it is essential to have a method for comparing them in order to reach conclusions. In this regard, the admissible orders for fuzzy numbers \citep{zumelzu2022} provide a reasonable tool for comparing them by a reflexive, anti-symmetric, transitive, and total binary relation.

In order to provide a better modeling of uncertainty, the Interval Type 2 Fuzzy Sets were introduced as mapping valuated on $\ci=\{x\in[0,1]^2\text{ : } x_1\leqslant x_2\}$.
\begin{definition}[Interval Type 2 Fuzzy Set]
    Let $X$ be the universe of discourse. An interval type 2 fuzzy set is a mapping $A=(\underline{A},\overline{A}):X\to\ci$.
\end{definition}

\subsection{Deck of Card- based Membership functions}

\noindent  Recently, the DoC method was applied to construct fuzzy numbers for decision-making processes \citep{DoC-MF}. This socio-technical method emphasizes the interaction between the decision-maker and the analyst, co-constructing the membership function to reflect the decision-maker’s preferences. The approach ensures that the resulting membership functions, termed DoC-based Membership Functions (DoC-MFs), are interpretable and tailored to the decision-maker's understanding. The construction process consists of three main steps \citep{DoC-MF}:
\begin{enumerate}
    \item \textit{Value Function Construction}. The original heterogeneous scales of the criteria (linguistic, discrete, interval, etc.) are transformed into a unified interval scale $[0,1]$, based on the decision-maker’s input using a deck of cards.
    \item \textit{Core and Support Definition}. The core and support are defined by eliciting the decision-maker's preferences through an interactive process.
    \item \textit{Left and Right-hand Sides}. The final step involves building the left and right-hand sides of the membership function based on further input from the decision-maker, reflecting varying degrees of confidence in the fuzzy set.
\end{enumerate}

Given a fuzzy number $A:\R\to[0,1]$, we can consider the points in which $A$ is not differentiable $\mathcal{D}^A=\{x\in\R\tq \nexists A^\prime(x)\}$, and its image set, namely $\mathcal{L}^A=\{\lim_{x\to x_0^+} A(x),\lim_{x\to x_0^-}A(x), A(x_0)\tq x_0\in\mathcal{D}^A\}$. At this point, let us introduce the formal definition of DoC-MF:
\begin{definition}[DoC-MF \cite{DoC-MF}]
    Let $A:\R\to[0,1]$ be a fuzzy number. We say that $A$ is a DoC-MF if:
    \begin{enumerate}
        \item[i)]$\mathcal{D}^A$ is a finite set,
        \item[ii)] For each $\alpha\in]0,1[$ there exist $\alpha_d,\alpha_u\in\mathcal{L}^A$ such that $\alpha_d<\alpha_u$, $]\alpha_d,\alpha_u[\cap\mathcal{L}^A=\emptyset$, $\alpha_d\leqslant \alpha\leqslant\alpha_u$ and
        \begin{align*}
            A_{\alpha}^-&=\frac{\alpha-\alpha_d}{\alpha_u-\alpha_d}(A_{\alpha_u}^--A_{\alpha_d}^-)+A_{\alpha_d}^-,\\
            A_{\alpha}^+&=\frac{\alpha-\alpha_d}{\alpha_u-\alpha_d}(A_{\alpha_u}^+-A_{\alpha_d}^+)+A_{\alpha_d}^+.
        \end{align*}
    \end{enumerate}
\end{definition}
In the original paper, the authors proved that DoC-MFs were completely determined by the $\alpha$-cuts corresponding to the values in $\mathcal{L}^A$ \citep{DoC-MF}. To do so, for a finite set $L\subset [0,1]$ containing $\{0,1\}$, they considered the mapping $\beta^{A,L}:\R\to L^2$ defined by:
\begin{equation*}
    \beta^{A,L}(x)=\begin{cases}
        (0,0)&\text{ if }x\notin A_0\\
        (1,1)&\text{ if }x\in A_1\\
        (\max\limits_{x\in A_\alpha}\{\alpha\in L \},\min\limits_{x\notin A_\alpha}\{\alpha\in L\})&\text{ if }x\in A_0\setminus A_1\\
    \end{cases}
\end{equation*}
with the notation $\beta^{A,L}(x)=(\beta_1^{A,L}(x),\beta_2^{A,L}(x))\pt x\in\R$ to refer to the coordinates of the mapping $\beta^{A,L}$.

\begin{theorem}\label{theo:charact}
    Let $A: \mathbb{R} \to [0, 1]$ be a DoC-MF, and consider a finite set $L$ of $\alpha$-levels such that $\cl^A \subseteq L$. Then, the mapping $A^L: \mathbb{R} \to [0, 1]$ defined by
    \[
    A^L(x) = \min \left\{ \frac{x - A^+_1}{A^+_2 - A^+_1}, \frac{x - A^-_1}{A^-_2 - A^-_1} \right\} (\beta^{A,L}_2(x) - \beta^{A,L}_1(x)) + \beta^{A,L}_1(x),
    \]
    where $[A^-_1, A^+_1] = A^{\beta^{A,L}_1(x)}$ and $[A^-_2, A^+_2] = A^{\beta^{A,L}_2(x)}$ for all $x \in \mathbb{R}$, satisfies $A^L(x) = A(x)$ for all $x \in \mathbb{R}$.
\end{theorem}
This allows defining the addition, product by scalars, and weighted average of DoC-MFs \cite{DoC-MF}.
 \begin{cor}[Addition of DoC-MFs]\label{cor:sum-pwlfn}
        Let $A,B:\R\to[0,1]$ be two DoC-MFs and consider a finite set $L$ of $\alpha$-levels such that $\mathcal{L}^A\cup \mathcal{L}^B\subset L$. Then, the mapping $(A+B)^L:\R\to\io$ defined by   
        \begin{equation*}
          \begin{aligned}
             (A+B)^L(x)=\min\left\{\frac{x-A_1^+-B_1^+}{A_2^++B_2^+-A_1^+-B_1^+},\;\,\frac{x-A_1^--B_1^-}{A_2^-+B_2^--A_1^--B_1^-}\right\}\cdot\\(\beta_2^{A\oplus B,L}(x)-\beta_1^{A\oplus B,L}(x))+\beta_1^{A\oplus B,L}(x)
         \end{aligned}
     \end{equation*}
     \noindent $\pt x\in\R$, where $[A_1^-,A_1^+]=A_{\beta_1^{A\oplus B,L}(x)}$, $[B_1^-,B_1^+]=B_{\beta_1^{A\oplus B,L}(x)}$, $[A_2^-,A_2^+]=A_{\beta_2^{A\oplus B,L}(x)}$, and $[B_2^-,B_2^+]=B_{\beta_2^{A\oplus B,L}(x)}\forall x\in\R$ satisfies $(A+B)^L(x)=(A\oplus B)(x)\pt x\in\R$.
    \end{cor}
      \begin{cor}[Product by scalars for DoC-MFs]\label{cor:prod-pwlfn}
        Let $A:\R\to[0,1]$ be a  DoC-MF and consider $r>0$ and a finite set $L$ of $\alpha$-levels such that $\mathcal{L}^A\subset L$. Then, the mapping $(rA)^L:\R\to\io$ defined by   \begin{equation*}
         (rA)^L(x)=\min\left\{\frac{x-rA_1^+}{rA_2^+-rA_1^+},\;\,\frac{x-rA_1^-}{rA_2^--rA_1^-}\right\}(\beta_2^{r\odot A,L}(x)-\beta_1^{r \odot A,L}(x))+\beta_1^{r \odot A,L}(x)
     \end{equation*}
     \noindent $\pt x\in\R$, where $[A_1^-,A_1^+]=A_{\beta_1^{r \odot A,L}(x)}$ and $[A_2^-,A_2^+]=A_{\beta_2^{r\odot A,L}(x)} \pt x\in\R$, satisfies $(rA)^L(x)=(r\odot A)(x)\pt x\in\R$.
    \end{cor}

   \begin{cor}[Weighted Average DoC-MFs]\label{corol:WA}
       Let $A^1,...,A^n:\R\to[0,1]$ be $n\in\N$ DoC-MFs such that $\sup(A^j)\subseteq\io\pt j=1,...,n$ and consider a finite set $L$ of $\alpha$-levels such that  $\mathcal{L}^{A^{j}}\subset L \pt j=1,...,n $. Additionally, let us consider the weights $w=(w_1,...,w_n)\in\io$ such that $\sum_{j=1}^nw_j=1$. Then, the mapping $\Phi_w(A^1,...,A^n)^L:\R\to\io$ defined by   \begin{equation*}
          \begin{aligned}
            \Phi_w(A^1,...,A^n)^L(x)=\min\left\{\frac{x-\sum_{j=1}^n w_j(A_1^j)^+}{\sum_{j=1}^n w_j(A_2^j)^+-\sum_{j=1}^n w_j(A_1^j)^+},\;\,\frac{x-\sum_{j=1}^n w_j(A_1^j)^-}{\sum_{j=1}^n w_j(A_2^j)^-\sum_{j=1}^n w_j(A_1^j)^-}\right\}\cdot\\(\beta_2^{\bigoplus\limits_{j=1}^n(w_j\odot A^j),L}(x)-\beta_1^{\bigoplus\limits_{j=1}^n(w_j\odot A^j),L}(x))+\beta_1^{\bigoplus\limits_{j=1}^n(w_j\odot A^j),L}(x)
         \end{aligned}
     \end{equation*}
     \noindent $\pt x\in\R$, where $[(A_1^j)^-,(A_1^j)^+]=A^j_{\beta_1^{\bigoplus\limits_{j=1}^n(w_j\odot A^j),L}(x)}$, $[(A_2^j)^-, A_2^j)^+]=A^j_{\beta_2^{\bigoplus\limits_{j=1}^n(w_j\odot A^j),L}(x)}\forall x\in\R$ satisfies $\Phi_w(A^1,...,A^n)^L(x)=\bigoplus\limits_{j=1}^n(w_j\odot A^j)(x)\pt x\in\R$ and $\sup(\Phi_w(A^1,...,A^n)^L)\subseteq\io$.
   \end{cor}

\section{DoC-T2MFs: A new way of building type-2 membership functions}\label{sec: DoCMembership}
\noindent This section introduces an innovative co-constructive approach for building type-2 membership function within multi-criteria decision-making framework. The basis of our approach is rooted in measurement theory of the building ratio scale and a revised version of DoC that allows incorporation of expert's hesitations. Before describing the building process of T2MF, we briefly introduce MCDM framework, theoretical background on the construction of memberships from ratio, and DoC for ration scale construction.

\subsection{Multicriteria decision analysis framework}
\noindent The multi-criteria decision-making (MCDM) framework is frequently used \citep{Figueira2005} to systematically address complex decision scenarios involving multiple perspectives, technically referred to as evaluation criteria, for assessing the courses of action under consideration. A typical MCDM scenario can be mathematically described by a collection of alternatives against the criteria $g_j$, represented as \({\cal{A}} = \{a_1, a_2, \ldots, a_m\}\), a finite set of \( n \) criteria, denoted as \({\cal{G}} = \{g_1, g_2, \ldots, g_n\}\), and the performance of the alternatives, expressed as \( g_j: {\cal{A}} \to E_{g_j} \) for each criterion \( g_j \) (\( j = 1, \ldots, n \)), where \( E_j \) represents the scale used to measure criterion \( g_j \). These scales \( E_{g_j} \) \((j=1,2,...,n)\) can be either ordinal or cardinal in nature.

Given this context, we assume that some scales are ordinal in nature, represented by a set of ordered levels or linguistic terms (if they are continuous we can discriteze them). These linguistic terms carry inherent subjectivity and ambiguity of human judgment, making their meaning (semantics) fuzzy by nature. This fuzziness can be modeled using type-1 or type-2 fuzzy sets. Here, we aim to outline the process of constructing one such scale by constructing type-2 membership functions associated with the levels (linguistic terms).

\subsection{Subjective ratio to membership function constructions}
\noindent The construction process follows the concept of subjective ratio estimation from measurement theory, specifically, focusing on the measurement of memberships \citep{Marchant2003}.

Let \( E_g = \{l_{g,1}, \ldots, l_{g,r}, \ldots, l_{g,k_{g}}\} \) be the set of levels used to evaluate a criterion \( g \). The levels in \( E_g \) can be interpreted as fuzzy sets relevant to the assessment of the criterion and the problem under consideration. Let \( X \) be the universal set on which the fuzzy sets of the levels are defined, and let \( X_{l_r} \subset X \) be the support of \( l_r \), i.e., the elements of \( l_r \) have a degree of belongingness to the fuzzy set. 

The subjective estimation of the ratio for the memberships of \( l_r \) can be defined by the mapping \( \rho_{l_r}: X \times X \to \mathbb{R}^+ \) such that \( (x, y) \mapsto \rho_{l_r}(x|y) \) for all \( (x, y) \in X \times X \). Here, the subjective ratio \( \rho_{l_r}(x|y) = b \) is interpreted as: "the ratio of the membership of \( x \) in \( l_r \) to the membership of \( y \) in \( l_r \) is \( b \)." For given subjective ratios for \( X_{l_r} \), the representation of the memberships \( \mu_{l_r} \) can be obtained using the following theorem:
\begin{theorem} \citep{Marchant2003}
Given the subjective ratios \( \rho_{l_r}: X \times X \to \mathbb{R}^+ \), if the following conditions are satisfied:
\begin{enumerate}
\item Reference Independence:  
    \( \rho_{l_r}(x|z) \geqslant \rho_{l_r}(y|z) \iff \rho_{l_r}(x|w) \geqslant \rho_{l_r}(y|w) \) for all \( x, y \in X \) and \( z, w \in X_{l_r} \).
\item Multiplicative Property:  
  \( \rho_{l_r}(x|z) = \rho_{l_r}(x|y) \cdot \rho_{l_r}(y|z) \), for all \( x, y, z \in X \).
\end{enumerate}
Then, there exists a membership function \( \mu_{l_r}: X \to [0,1] \) such that:
\begin{itemize}[label={--}]
\item \( \mu_{l_r}(x) \geqslant \mu_{l_r}(y) \iff \rho_{l_r}(x|z) \geqslant \rho_{l_r}(y|z) \), \( \forall x, y \in X \) and \( z \in X_{l_r} \),
\item \( \frac{\mu_{l_r}(x)}{\mu_{l_r}(y)} = \rho_{l_r}(x|y) \), \( \forall x, y \in X \).
\end{itemize}
Furthermore, for an arbitrary element \( w \in X_{l_r} \), the membership function \( \mu_{l_r} \) is given by:  
\[
\mu_{l_r}(x) = \frac{\rho_{l_r}(x|w)}{\max_{y \in X} \{\rho_{l_r}(y|w)\}}.
\]
\end{theorem}

\subsection{New version of Deck of cards for membership function construction}\label{sec: new_DoC}
\noindent Let us consider a fuzzy set $A$ and a set of points, say,  $\{x_1,\dots x_r, \dots, x_p\}$, from the support of the fuzzy set $A$. We will refer to points as alternatives. Our objective is to determine the membership values of each alternative, $A(x_1),\dots A(x_r),\dots, A(x_r), \ldots, A(x_p)$. These values are on a ratio scale and subjective ratios can be elicited from decision-makers or experts. We use a variant of the deck of card method (DCM) by \cite{corrente2021} as follows:
\begin{enumerate}
    \item Let us assume that 
    \begin{itemize}[label={--}]
        \item The alternatives, $x_1,\dots x_r, \dots, x_p$, belong to to the right part of the membership function. 
        \item Alternative $x_1$ belongs to the support of the fuzzy set $A$. Its membership value is a very small positive value, say, $\epsilon$, i.e., $A(x_1) = \epsilon$.
        \item Alternative $x_p$ belongs to the core of the $A$. Thus, $A(x_p)=1$.
        \item All the alternatives $x_1,\dots x_r, \dots, x_p$ have a value in an interval scale, as suggested by \cite{GARCIAZAMORA2024108863}, $v(x_1),\dots,v(x_r),\dots v(x_p)$. These values are within the range $[0,1]$
        \item All the alternatives are on the right of the core of the fuzzy set $A$. We will show how to compute for the alternatives on the right, the procedure for the left is similar.
        \begin{figure}[ht!]
            \centering
            \includegraphics[width=0.35\linewidth]{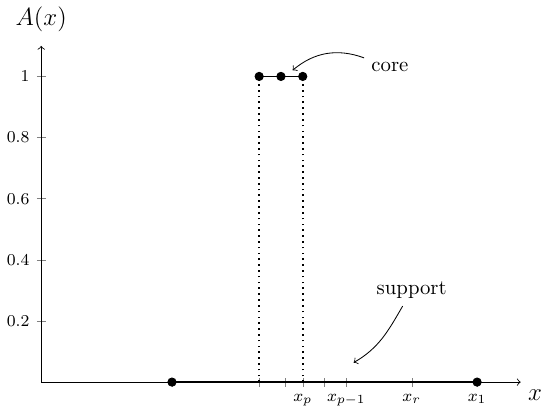}
            \caption{Support and Core of A}
            \label{fig: T2MF membership_support_core}
        \end{figure}
    \end{itemize}
    \item We want to elicit the subjective ratios $\frac{A(x_p)}{A(x_{p-1})}, \frac{A(x_p)}{A(x_{p-2})}, \dots, \frac{A(x_2)}{A(x_1)}$. The comparison of $A(x_1)$ with the rest is insignificant in the sense that $A(x_1)$ belongingness to the fuzzy set is almost negligible. 
    \item 
    The analyst provides the decision-maker with a set of blank cards to model how the decision-makers feels about this difference in the membership values of the alternatives. Subsequently, the analyst asks the decision-maker to place the cards between consecutive alternatives based on their subjective judgment of the differences in membership. The decision-maker adds sets of blank cards by hand and represents her/his feelings about the intensity of belief difference between two consecutive alternatives belongingness to the fuzzy set $A$. The analyst then counts the number of cards placed by the decision-maker between each pair of consecutive alternatives. Let \( e_r \) denote the number of cards placed between the consecutive alternatives \( x_{r} \) and \( x_{r+1} \) for \( r = 1, 2, \dots, p-1 \). The memberships' judgment provided by the decision-maker can be represented as the following chain\footnote{Note that this procedure can also be used to determine the non-normalized weights $\bar{w}_h, (h=1,\dots p)$ of the criteria $g_j,\; (h=1,\dots p)$ in the context of determining their relative importance in MCDA/M, mainly in outranking methods, by setting up them as follows: $\bar{w}_p = 0, \quad \bar{w}_{p-1} = \bar{w}_p + (e_{p-1}+1), \quad \dots, \quad \bar{w}_1 = \bar{w}_2 + (e_1 + 1)$. Next, we compute the sum: $W = \sum_{h=1}^{p} \bar{w}_h$.
    Finally, the normalized weights $w_h,(h=1,\dots,p$) are obtained as follows: $w_h = \frac{\bar{w}_h}{W},\; h=1,\dots,p.$ (in this example we assume we do not have \textit{ex \ae quo} criteria, and that criterion $g_p$ is the one ``with no relevance''). Let us remember that $0$ blank cards in between two criteria means the difference between their weights is minimal (the unit), one blank card is twice the unit, and so on. Example: $g_4\;[0]\;g_3\;[2]\;g_2\;[1]\;g_1$, leading to $\bar{w}_4 = 0$, $\bar{w}_3 = 1$, $\bar{w}_2 = 4$, and $\bar{w}_1 = 6$ ($0\;[0+1]\;1\;[2+1]\;4\;[1+1]\; 6$). We have $W= 11$, and the corresponding normalized weights: $w_4 = 0$, $w_3 = 1/11 = 0.091$, $w_2 = 4/11 = 0.364$, and $w_1 = 6/11 = 0.545$, where the sum is equal to $1$.}:


    $$\{x_1\}\; [e_1] \; \{x_2\}, \dots, \{x_r\}\; [e_r]\; \{x_{r+1}\} \dots \{x_{p-1}\}\; [e_{p-1}] \{ x_{p} \}$$.

    Further, the decision-maker's membership judgments over the consecutive levels via cards can be presented through a pairwise comparison matrix given in Figure. \ref{fig:blank-cards_general}.
    \begin{center}
    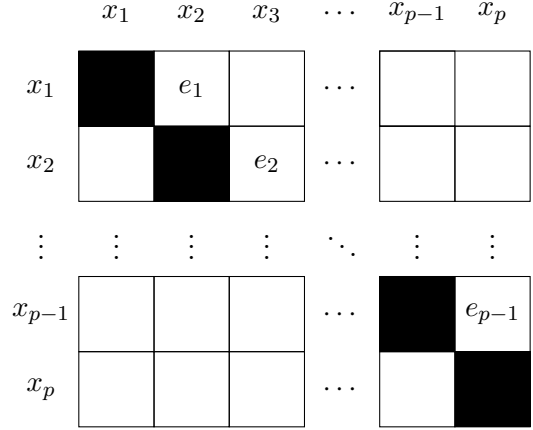
\begin{figure}[ht!]
        \centering
        \begin{tikzpicture}
            \def\cellsize{1}
            \node at (0.5, 0.5) {$x_1$};
            \node at (1.5, 0.5) {$x_2$};
            \node at (2.5, 0.5) {$x_3$};
            \node at (3.5, 0.5) {$\cdots$};
            \node at (4.5, 0.5) {$x_{p-1}$};
            \node at (5.5, 0.5) {$x_p$};
            \node at (-0.5, -0.5) {$x_1$};
            \node at (-0.5, -1.5) {$x_2$};
            \node at (-0.5, -2.5) {$\vdots$};
            \node at (-0.5, -3.5) {$x_{p-1}$};
            \node at (-0.5, -4.5) {$x_p$};
            \fill[black] (0, 0) rectangle (1, -1); 
            \draw (0, 0) rectangle (1, -1); \node[text=white] at (0.5, -0.5) {};
            \draw (1, 0) rectangle (2, -1); \node at (1.5, -0.5) {$e_1$};
            \draw (2, 0) rectangle (3, -1);\node at (2.5, -0.5) {};
           \node at (3.5, -0.5) {$\cdots$};
            \draw (5, 0) rectangle (4, -1); \node at (5.5, -0.5) {};
            \draw (6, 0) rectangle (4, -1); \node at (6.5, -0.5) {};
            \draw (0, -1) rectangle (1, -2); \node at (0.5, -1.5) {};
            \fill[black] (1, -1) rectangle (2, -2); 
            \draw (1, -1) rectangle (2, -2); \node[text=white] at (1.5, -1.5) {};
            \draw (2, -1)  rectangle (3, -2); \node at (2.5, -1.5) {$e_2$};
            \node at (3.5, -1.5) {$\cdots$};
            \draw (4, -1) rectangle (5, -2); \node at (5.5, -1.5) {};
            \draw (5, -1) rectangle (6, -2); \node at (6.5, -1.5) {};
            \node at (0.5, -2.5) {\vdots};
            \node at (1.5, -2.5) {\vdots};
            \node at (2.5, -2.5) {\vdots};
            \node at (3.5, -2.5) {$\ddots$};
            \node at (4.5, -2.5) {\vdots};
            \node at (5.5, -2.5) {\vdots};
            \draw (0, -3) rectangle (1, -4); \node at (0.5, -3.5) {};
            \draw (1, -3) rectangle (2, -4); \node at (1.5, -3.5) {};
            \draw (2, -3) rectangle (3, -4); \node at (2.5, -3.5) {};
             \node  at (3.5, -3.5) {$\cdots$};
            \fill[black](4,-3) rectangle (5, -4); 
            \draw (4, -3) rectangle (5, -4); \node[text=white] at (4.5, -3.5) {};
            \draw (5, -3) rectangle (6, -4); \node at (5.5, -3.5) {$e_{p-1}$};
        
            \draw (0, -4) rectangle (1, -5); \node at (0.5, -4.5) {};
            \draw (1, -4) rectangle (2, -5); \node at (1.5, -4.5) {};
            \draw (2, -4) rectangle (3, -5); \node at (2.5, -4.5) {};
             \node  at (3.5, -4.5) {$\cdots$};
            \draw (4, -4) rectangle (5, -5); \node at (4.5, -4.5) {};
            \fill[black] (5, -4) rectangle (6, -5); 
            \draw (5, -4) rectangle (6, -5); \node[text=white] at  (5.5, -4.5) {};
                 
        \end{tikzpicture}
         \caption{Blank Cards}
        \label{fig:blank-cards_general}
    \end{figure}
    \end{center}

    To interpret the decision-maker's memberships elicitation further, let us consider the example in Fig.\ref{fig:blank-cards_example} with three alternatives $\{x_1,x_2,x_3\}$ and decision-maker inserts $1$ card between $x_2$ and $x_1$, and $4$ cards between $x_3$ and $x_2$. The example tells that the difference in the membership values when moving from $x_3$ to $x_1$ is greater than the one from $x_2$ to $x_1$ and this can be modelled by inserting 4 cards and 1 card, in the intervals between the first two pairs of alternative and the second pair.
    \begin{center}
    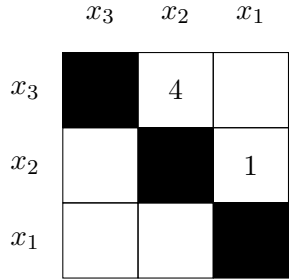
\begin{figure}[ht!]
    \centering
    \begin{tikzpicture}
        \def\cellsize{1}
        \node at (0.5, 0.5) {$x_3$};
        \node at (1.5, 0.5) {$x_2$};
        \node at (2.5, 0.5) {$x_1$};
        
        \node at (-0.5, -0.5) {$x_3$};
        \node at (-0.5, -1.5) {$x_2$};
        \node at (-0.5, -2.5) {$x_1$};
        
        \fill[black] (0, 0) rectangle (1, -1); 
        \draw (0, 0) rectangle (1, -1); \node[text=white] at (0.5, -0.5) {};
        \draw (1, 0) rectangle (2, -1); \node at (1.5, -0.5) {4};
        \draw (2, 0) rectangle (3, -1); \node at (2.5, -0.5) {};
       
        \draw (0, -1) rectangle (1, -2); \node at (0.5, -1.5) {};
        \fill[black] (1, -1) rectangle (2, -2); 
        \draw (1, -1) rectangle (2, -2); \node[text=white] at (1.5, -1.5) {};
        \draw (2, -1) rectangle (3, -2); \node at (2.5, -1.5) {1};
       
        \draw (0, -2) rectangle (1, -3); \node at (0.5, -2.5) {};
        \draw (1, -2) rectangle (2, -3); \node at (1.5, -2.5) {};
        \fill[black] (2, -2) rectangle (3, -3); 
        \draw (2, -2) rectangle (3, -3); \node[text=white] at (2.5, -2.5) {}; 
       \end{tikzpicture}
      \caption{An example of blank cards insertion}
        \label{fig:blank-cards_example}
    \end{figure}
    \end{center}

    \item Compute the non-normalized values. From this information provided by the decision-maker, the non-normalized values of memberships of the alternatives in $A$, $\bar{A}: \{x_1,...,x_p\} \to [0, \infty) $  can be determined from the following relations:
    \[ \bar{A}(x_r) = \bar{A}(x_1) + \sum_{h=1}^{r-1} (e_h+1),\; \text{for}\; r=2, \dots p,\]
    where $\bar{A}(x_1)$, the non-normalized value of the right endpoint of the support is zero, i.e., $\bar{A}(x_1)=0$. 
    \item Compute the ratios. Up to this stage, the decision-maker provides only indirect mmebership judgments information regarding the alternatives' memberships and may not have a clear understanding of how their judgments are represented on the ratio scales. To ``confirm'' that DM's judgments has been adequatly modelled in the ratio scale, we compute the ratios between the membership values and present the following tables to the decision-maker:
    \begin{center}
    \begin{figure}[ht!]
        \centering
        \begin{tikzpicture}[scale=1]
            \def\cellsize{2}
            \node at (0.5, 0.5) {$x_p$};
            \node at (1.5, 0.5) {$x_{p-1}$};
            \node at (2.5, 0.5) {$x_{p-2}$};
            \node at (3.5, 0.5) {$\cdots$};
            \node at (4.5, 0.5) {$x_3$};
            \node at (5.5, 0.5) {$x_2$};
            \node at (-0.5, -0.5) {$x_p$};
            \node at (-0.5, -1.5) {$x_{p-1}$};
            \node at (-0.5, -2.5) {$\vdots$};
            \node at (-0.5, -3.5) {$x_{3}$};
            \node at (-0.5, -4.5) {$x_2$};
            \fill[black] (0, 0) rectangle (1, -1); 
            \draw (0, 0) rectangle (1, -1); \node[text=white] at (0.5, -0.5) {1};
            \draw (1, 0) rectangle (2, -1); \node at (1.5, -0.5) {$a^p_{p-1}$};
            \draw (2, 0) rectangle (3, -1);\node at (2.5, -0.5) {$a^p_{p-2}$};
           \node at (3.5, -0.5) {$\cdots$};
            \draw (5, 0) rectangle (4, -1); \node at (4.5, -0.5) {$a^p_{3}$};
            \draw (6, 0) rectangle (4, -1); \node at (5.5, -0.5) {$a^p_{2}$};
            \draw (0, -1) rectangle (1, -2); \node at (0.5, -1.5) {};
            \fill[black] (1, -1) rectangle (2, -2); 
            \draw (1, -1) rectangle (2, -2); \node[text=white] at (1.5, -1.5) {1};
            \draw (2, -1)  rectangle (3, -2); \node at (2.5, -1.5) {$a^{p}_{p-2}$};
            \node at (3.5, -1.5) {$\cdots$};
            \draw (4, -1) rectangle (5, -2); \node at (4.5, -1.5) {$a^{p-1}_{3}$};
            \draw (5, -1) rectangle (6, -2); \node at (5.5, -1.5) {$a^{p-1}_{2}$};
            \node at (0.5, -2.5) {\vdots};
            \node at (1.5, -2.5) {\vdots};
            \node at (2.5, -2.5) {\vdots};
            \node at (3.5, -2.5) {$\ddots$};
            \node at (4.5, -2.5) {\vdots};
            \node at (5.5, -2.5) {\vdots};
            \draw (0, -3) rectangle (1, -4); \node at (0.5, -3.5) {};
            \draw (1, -3) rectangle (2, -4); \node at (1.5, -3.5) {};
            \draw (2, -3) rectangle (3, -4); \node at (2.5, -3.5) {};
             \node  at (3.5, -3.5) {$\cdots$};
            \fill[black](4,-3) rectangle (5, -4); 
            \draw (4, -3) rectangle (5, -4); \node[text=white] at (4.5, -3.5) {1};
            \draw (5, -3) rectangle (6, -4); \node at (5.5, -3.5) {$a^3_{2}$};
        
            \draw (0, -4) rectangle (1, -5); \node at (0.5, -4.5) {};
            \draw (1, -4) rectangle (2, -5); \node at (1.5, -4.5) {};
            \draw (2, -4) rectangle (3, -5); \node at (2.5, -4.5) {};
             \node  at (3.5, -4.5) {$\cdots$};
            \draw (4, -4) rectangle (5, -5); \node at (4.5, -4.5) {};
            \fill[black] (5, -4) rectangle (6, -5); 
            \draw (5, -4) rectangle (6, -5); \node[text=white] at  (5.5, -4.5) {1};      
        \end{tikzpicture}
         \caption{Ratio Table}
        \label{fig:ratio_table_general}
    \end{figure}
    \end{center}
where $a^s_r =\bar{A}(x_s)/ \bar{A}(x_r)$ for all $s,r=p,\dots 2$ and $s<r$.
     Afterward the analyst tries to know the decision-maker's preference on the ratio table in the following way:
     \begin{itemize}
         \item[$-$]  the analyst asks the decision-maker whether she/he feels satisfied with the ratios between membership values of the different alternatives. If the answer is YES. The membership values of the alternatives are computed from non-normalized values.
         \item[$-$] If the answer was ``No!". The analyst asks the DM to modify the ratios for which she/he is not satisfied. Let us assume that DM is not satisfied with the ratio of $x_p$ to $x_2$ and provide the adjusted ratio $\bar{a}^p_2$. We will denote other entries of the modified tables as $\bar{a}^s_r=a^s_r, s,r=p,\dots 2, s<r$ and $(s,r) \neq (p,2)$.
     \end{itemize}
    \item Compute the normalized membership values. The normalized membership values $A: \{x_1,\dots, x_p\} \to [0,1]$ of the alternatives are computed as follows:
    \[A(x) = \frac{\bar{A}(x_r)}{\bar{A}(x_p)},\; \text{for}\; r=1,2,\dots,p.\]
    \item Adjustments. Given the adjusted ratio table, the DM may want to know a new compatible set of cards need to insert between consecutive alternatives. The key idea is to check the membership differences are adequate to the DM judgments. For that purpose, we first compute the approximate revised non-normalized values $\bar{A}^m(x_r),(r=2,\dots p)$ from the following optimization model
    \begin{equation*}
        \begin{split}
        \min & \hskip 5pt \sum_{s,r=2, r<s}^p \vert \bar{A}^m(x_r) - \bar{a}^s_r\bar{A}^m(x_s) \vert \\
        \text{s.t.} & \hskip 5pt\begin{cases}
        \bar{A}^m(x_r) \geqslant 1, \; r=2,...,p\\
        \end{cases}
        \end{split}
    \end{equation*}
    The idea is to obtain the ``best fit'' for the judgments provided by the DMs by deviating minimally from the modified ratio (alternative models could be used leading to difference compatible sets of blank cards). From the modified non-normalized values $\{\bar{A}^m(x_1), \dots \bar{A}^m(x_p)\}$, we determine the number of cards that DM needs to inserts between the consecutive alternatives $\{x_1,\dots, x_p\}$, i.e., the chain, 
    \[ \{x_1\} \; [\bar{e}_1] \{x_2\} \; [\bar{e}_2] \dots [\bar{e}_{p-1}]\; \{x_p\}\]
    The problem could be viewed as an inverse problem of Deck of cards, i.e., assuming a sequence generated by the deck of cards we determine the number of blank cards the decision-maker could put between consecutive levels. Further, the exact new sequence of cards $\{\bar{e}_1,\dots \bar{e}_p\}$ could be found by solving the following integer linear programming problem
    \begin{equation*}
        \begin{split}
            \text{min} & \hskip 5pt \sum_{r=1}^{p-1}\left\vert \frac{(\bar{A}^m(x_{r})-\bar{A}^m(x_{r+1}))h}{\bar{A}^m(x_{2})(z-1)}-(\bar{e}_{r}+1)\right\vert \\
            \text{s.t.} & \hskip 5pt\begin{cases}
            \sum_{r=1}^{p-1}(\bar{e}_r+1)=h\\
            \bar{e}_r \in \mathbb{N}\cup \{0\}, r=1,\cdots, p-1\\
            h \in \mathbb{N}
            \end{cases}
        \end{split}
    \end{equation*}

 
    \item Final. The analyst presents the decision maker's the new ratio table obtained based on the modified non-normalized values $\{\bar{A}^m(x_1), \dots, \bar{A}^m(x_p)\}$ and asked the DM's whether she/he is satisfied of it or not
    \begin{itemize}
        \item[-] If the DM's answer is YES!  The membership values of alternatives in $A$ is computed by normalizing $\{\bar{A}^m(x_1), \dots, \bar{A}^m(x_p)\}$, i.e.,  \[A(x_r) = \frac{\bar{A}^m(x_r)}{\bar{A}^m(x_p)},\; \text{for}\; r=1,2,\dots,p.\]
        \item[-] If DM's answer is NO! The analyst will repeat the process from Step 5 to elicit the DM's preference.
    \end{itemize}
    In this way, the analyst could elicit the DM's membership judgments and construct the right-hand membership values of the alternatives $\{x_1, \dots, x_p \}$ in the fuzzy concept $A$ (Figure \ref{fig: T1MF membership}). Similarly, left-hand membership function of the fuzzy concepts $A$ can be elicited through the interaction of decision analyst and decision-maker. 
     \begin{figure}[htpb!]
        \centering
        \includegraphics[width=0.4\linewidth]{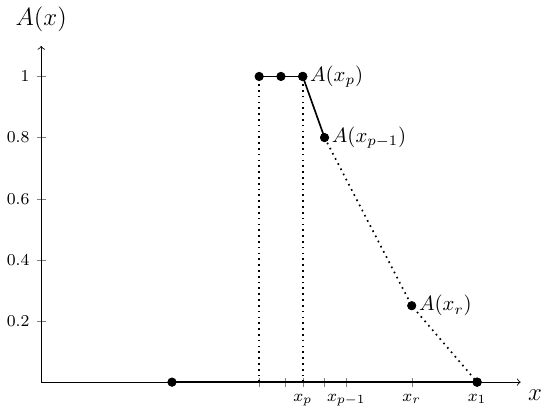}
        \caption{Right-hand membership function in ratio scale}
        \label{fig: T1MF membership}
    \end{figure} 
\end{enumerate}

\subsection{Construction of DoC based IT2MF}
\noindent  In this section, we develop the three-step methodology to construct IT2MF based on our membership function construction method proposed in Section \ref{sec: new_DoC}. Specifically, we construct a linguistic scale in which each linguistic label has a fuzzy meaning and is represented by an IT2MF. The construction process is guided by a decision analyst and starts with the transformation of the linguistic labels into a value scale through a co-constructive DoC method. Based on these the core and support of IT2MF representation of different labels are identified in the second step. In the final step, we capture the DM's hesitation in membership degree allocation by allowing imprecision in providing the blank cards and constructing the left and right-hand parts of IT2MF.

\subsubsection{Value function construction for labels}
\noindent  Following the information articulation protocol of DoC method described in Section \ref{sec: new_DoC}, we build a value function for the linguistic labels using a version of the DoC illustrated by \cite{GARCIAZAMORA2024108863}. Let \( E_g = \{l_{g,1}, \ldots, l_{g,r}, \ldots, l_{g,k_{g}}\} \) be the labels of linguistic scale such that $l_{g,1}\; \prec, \ldots \prec \; l_{g,r} \prec \; \ldots \prec \; l_{g,k_{g}}$. Based on the information of blank cards provided by DM's between labels
\[l_{g,1}\; [e_1] \; l_{g,2}\ldots  \; l_{g,r-1} \; [e_{r}]  \; l_{g,r}\; \ldots \;  [e_{k_g -1}] \;l_{g,k_{g}} \]
and values of the labels $v_g: E_g \to [0,1]$, with the assumptions on the worst and best labels values $v_g(l_{g,1})=0$ and $v_g(l_{g,k_{g}})=1$, we obtain value of the linguistic labels as follows:



\begin{equation*}
        v_g(l_{g,p}) = v_g(l_{g,1}) + \gamma \sum_{h=1}^{p-1} (e_h+1), \; p=2,\dots, k_g-1
    \end{equation*}
where $\gamma$ the value of a blank card. It can be computed as  $\gamma = \frac{v_g(l_{g,k_g})-v_g(l_{g,1})}{r}$, with $r$ being the total number of blank cards, $r =\sum_{h=1}^{k_g-1}(e_h+1)$ that are placed between the worst linguistic labels and best linguistic labels. Such a value function $v_g:E_g \to [0,1]$ is depicted in Fig. \ref{fig: value_function_labels}.
 \begin{figure}[htpb!]
        \centering
        \includegraphics[width=0.5\linewidth]{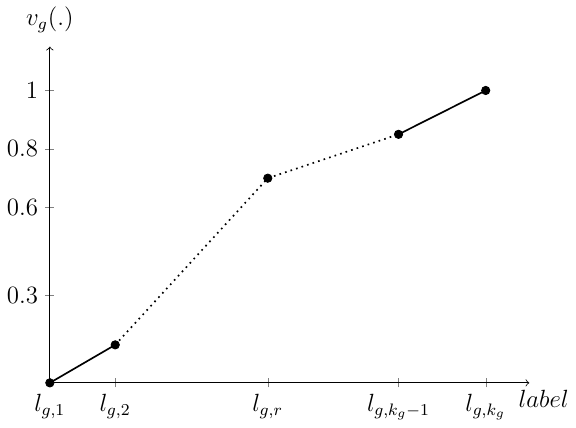}
        \caption{Value function for label of $E_g$}
        \label{fig: value_function_labels}
    \end{figure}
It is evident from Figure \ref{fig: value_function_labels} that DoC method approximates the value function in the mind of DM through a continuous piecewise linear function. In fact, DoC method can approximate any arbitrary non-linear function by a piecewise continuous linear function with a desirable error bound controlled by the total number of blank cards used in the preference elicitation process and that has been demonstrated in Theorem \ref{th: num_to_Card} (see Appendix).
\subsubsection{Identifying core and support of IT2MF}
\noindent The core of the IT2MF indicates an interval on the domain of the membership function where DM's has full confidence (no uncertainty, i.e., membership grade is $1$) in stating that values in the region represent the linguistic label. On the other hand, support of the IT2MF is an interval where DM's feels that he/she has somewhat confidence. Further, these can be identified through the co-constructive process of identifying the threshold proposed in \cite{GARCIAZAMORA2024108863}, where the analyst asks the DM's about his/her confidence by presenting a value and keeps on changing based on the provided answers. Suppose that core and support of one such linguistic level $l_{g,r}$ from $E_g$ is constructed by following the above process and depicted in Fig. \ref{fig: core_support_l_r}.
 \begin{figure}[htpb!]
        \centering \includegraphics[width=0.4\linewidth]{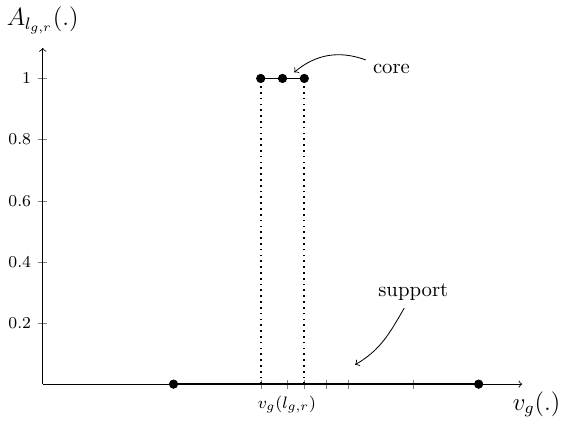}
        \caption{Core and support of the fuzzy concept $l_{g,r}$}
        \label{fig: core_support_l_r}
    \end{figure}
\subsubsection{Identify left-right upper and lower membership function}
\noindent  Now, we attempt to construct the left-right upper and lower membership functions part of IT2MF $A_{l_{g,r}} = (\underline{A}_{l_{g,r}}, \bar{A}_{l_{g,r}})$ by coupling the idea of constructing T1MF proposed in Section 3.1 with the interpersonal uncertainty appears in the human subjective judgments. The construction phase consists of two steps. In the first step, we attempt to capture the uncertainty in the subjective judgment of DM's that subsequently leads to allocating several membership degrees for elements in the support, which could also be viewed as several T1MFs generated from DM's interpersonal uncertainty. In the second step, IT2MF is constructed from these T1MFs using some simple aggregation rule. We will describe the construction of the right-hand upper and lower membership functions by considering a similar set-up in T1MF, i.e., we want to determine the imprecise membership degrees of a set of points $\{x_1,..,x_r,...,x_p\}$ at the right hand of the core of $A_{l_{g,r}}$. The steps are described as follows:
\begin{itemize}[label={--}]
  \item \emph{Construction of T1MFs from hesitation}. In T1MF construction, we have considered the DM's subjective judgment via a set of blank cards to obtain the membership values of the alternatives to the fuzzy concept $A$. Such subjective judgments could model the precise membership value of the fuzzy set (T1MF).
  However, such precise judgments in the form of blank cards may not be easy to elicit when DM has intrapersonal uncertainty/hesitation. In such cases, DM cannot provide the precise number of blank cards included between two consecutive alternatives, say $x_r$ and $x_{r+1}$. In these scenarios, DM prefers to express the uncertainty associated with the number of blank cards by an interval of possible values. Let $e_r = [e_r^l, e_r^u], (e_r^l \leqslant e_r^u)$ be the uncertain interval that represents the DM imprecise judgments in the number of blank cards between the alternatives $x_r$ and $x_{r+1}$. Note that DM uses such interval estimation, where they hesitate to insert blank cards, not necessarily between all the consecutive alternatives. To understand the impact of inserting interval estimation of blank cards between levels, let's assume that the DM insert an interval estimation of the blank cards between the alternatives $x_{p-1}$ and $x_p$, i.e., $e_{p-1} = [e_{p-1}^l, e_{p-1}^u]$ and subsequently, the chain of the alternatives and DM's membership judgments in the form of blank cards can be represented as the following chain:
     $$\{x_1\}\; [e_1] \; \{x_2\}, \dots, \{x_r\}\; [e_r]\; \{x_{r+1}\} \dots \{x_{p-1}\}\; [e_{p-1}^l,e_{p-1}^u] \{ x_{p} \} $$

     This asserts that DM cannot precisely determine the difference between the belonging of alternatives $x_{p-1}$ and $x_p$ the fuzzy set $A$ but she/he can provide an interval estimation $[e_{p-1}^l,e_{p-1}^u]$ on the number of blank cards. In fact, the number of blank cards could be any integer from $[e_{p-1}^l,e_{p-1}^u]$. Thus, we obtain the following set of chains of the alternatives and blank cards:
     \[\mathcal{C}=\{\{x_1\}\; [e_1] \; \{x_2\}, \dots, \{x_r\}\; [e_r]\; \{x_{r+1}\} \dots \{x_{p-1}\}\; [e_{p-1}] \{ x_{p} \} : e_{p-1} \in [e_{p-1}^l,e_{p-1}^u]\}.
     \]

    Further, for each chain from $\mathcal{C}$, we can obtain membership values of alternatives in the fuzzy set $A$ by following the similar process described in Steps 2-8. Subsequently, we obtain the different values of the membership for each alternative $x_k, (k=1,\dots p)$ in the fuzzy set $A$. Thus, the imprecise membership judgments of DM's via interval estimation of the number of cards
    propagate the uncertainty into the primary membership values of the alternatives and eventually, we obtain several representations of the right-hand membership function for the fuzzy set $A$ depicted in Fig. \ref{fig: T2MF membership construction_1}.  
     \begin{figure}[htpb!]
        \centering
        \includegraphics[width=0.5\linewidth]{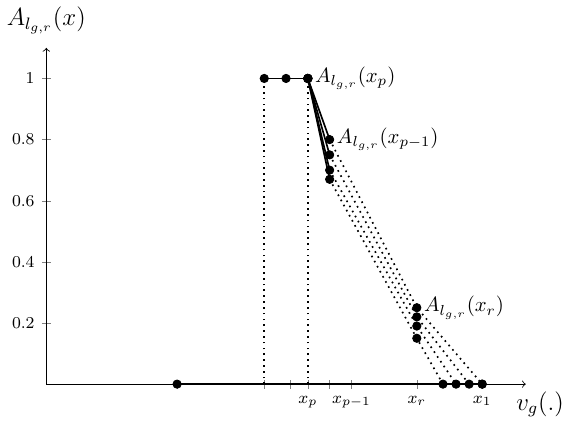}
        \caption{Right-hand membership function construction from uncertainty into the number of blank cards}
        \label{fig: T2MF membership construction_1}
    \end{figure} 
    \item \emph{Construction of IT2MF from T1MFs}: We observe that hesitation in giving blank cards leads to uncertainty in the membership values allocation for each alternative. In fact, for each alternative $x_k$, we obtain a set of membership values for $A_{l_{g,r}}(x_k)$ that characterize the uncertainty in allocating the membership values through imprecise blank cards. The upper and lower bounds of this set of membership values provide the DM's uncertainty about the membership judgments for the alternative $x_k$'s belongingness, i.e., the membership value of the alternative $x_k$ is not a precise but an interval that represents DM's uncertainty in the belongings of $x_k$ in the fuzzy concept $A_{l_{g,r}}$. Considering this fact for all the alternatives, we obtain the right-hand upper ($\underline{A}_{l_{g,r}}$) and lower ($\overline{A}_{l_{g,r}}$) membership functions for the fuzzy set $A$ as depicted in Fig. \ref{fig: T2MF membership construction_2}. In this way, we can elicit the upper and lower right-hand membership functions of IT2FS. Similarly, the left-hand side of the IT2FS could be elicited by the DM's and subsequently the IT2MF is constructed and depicted as in Fig. \ref{fig: T2MF membership construction_3}. 
    
    \begin{figure}[htpb!]
     \begin{subfigure}[t]{0.48\textwidth}
        \centering
        \includegraphics[width=1\linewidth]{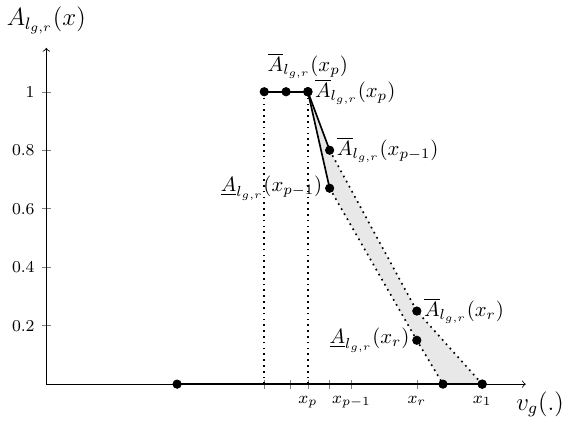}
        \caption{Right-hand upper and lower membership functions construction from uncertainty into the number of blank cards}
        \label{fig: T2MF membership construction_2}
    \end{subfigure} 
    \begin{subfigure}[t]{0.48\textwidth}
        \centering
        \includegraphics[width=1\linewidth]{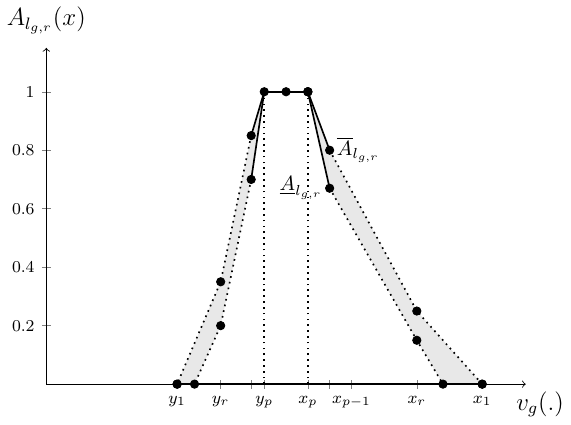}
        \caption{IT2MF constructed using modified DoC method}
        \label{fig: T2MF membership construction_3}
    \end{subfigure} 
    \caption{DoC-IT2MF from DoC-T1MFs}
    \end{figure}
    
\end{itemize}

\section{Aggregations and ordering rules for DoC-IT2MF }\label{sec: Aggregation Functions}
\noindent  In this section, we first introduce the representation of DoC-IT2MF using DoC-T1MF as a foundation. Building on this, we develop an aggregation rule for DoC-IT2MFs and establish an ordering principle for these fuzzy quantities.  
\begin{definition}[Interval Type 2 Fuzzy number]
    An interval type 2 fuzzy nuber is a mapping $A=(\underline{A},\overline{A}):\R\to\ci$ such that:
    \begin{enumerate}
        \item $\underline{A},\overline{A}$ are fuzzy numbers.
        \item $\underline{A}(x)\leqslant\overline{A}(x)\pt x\in \R$
    \end{enumerate}
\end{definition}
Note that $\underline{A}_1\subseteq\overline{A}_1$. 

\begin{definition}[DoC-IT2MF] 
    A DoC-IT2MF is a mapping $A=(\underline{A},\overline{A}):\R\to\ci$ such that:
    \begin{enumerate}
        \item $\underline{A},\overline{A}$ are DoC-MFs
        \item $\underline{A}(x)\leqslant\overline{A}(x)\pt x\in \R$
    \end{enumerate}
\end{definition}
Note that any DoC-IT2MF is a Type 2 fuzzy set. $\cd_2$ is the set of all the DoC-IT2MFs on $\io$.
\begin{definition}
    Let $A=(\underline{A},\overline{A})$. Then, given $\alpha\in\iol$, we can define its $\alpha$-cut by $A_\alpha=(\underline{A}_\alpha,\overline{A}_\alpha)$. Of course, $\underline{A}_\alpha\subseteq\overline{A}_\alpha \pt\alpha\in\iol$
\end{definition}

\begin{theorem}\label{theo:Sumprodt2}
    Let $A=(\underline{A},\overline{A})$, $B=(\underline{B},\overline{B})\in\cff$, $\lambda>0$. Then, $A\oplus B=(\underline{A}\oplus\underline{B},\overline{A}\oplus\overline{B})$ and $\lambda\odot A=(\lambda\odot\underline{A},\lambda\odot\overline{A})$ are IT2FN. Furhtermore, for $\alpha\in\iol$, $(A\oplus B)_\alpha=(\underline{A}_\alpha+\underline{B}_\alpha,\overline{A}_\alpha+\overline{B}_\alpha)$ and $(\lambda\odot A)_\alpha=(\lambda\underline{A}_\alpha,\lambda\overline{A}_\alpha)$.
\end{theorem}
\begin{proof}
    Fuzzy arithmetic guarantees that $\underline{A}\oplus\underline{B},\overline{A}\oplus\overline{B}$ and $\lambda\odot\underline{A},\lambda\odot\overline{A}$ are fuzzy numbers. Now, note that
    \begin{equation*}
        \begin{split}
        \underline{A}\oplus\underline{B}(z)&=\sup_{z=x+y}\min\{\underline{A}(x),\underline{B}(y)\}\leqslant\sup_{z=x+y}\min\{\overline{A}(x),\overline{B}(y)\}=\overline{A}\oplus\overline{B}(z)\\
        \lambda\odot\underline{A}(z)&=\sup_{z=\lambda x}\underline{A}(x)\leqslant\sup_{z=\lambda x}\overline{A}(x)=       \lambda\odot\overline{A}
        \end{split}
    \end{equation*}
    for any $z\in\io$, which implies that both $A\oplus B$ and $\lambda\odot A$ are IT2FNs and also determines the shape of the $\alpha$-cuts. 
\end{proof}

Taking this into account, we can derive some computational rules for DoC-IT2MFs. For a DoC-IT2MF $A=(\underline{A},\overline{A}):\R\to\ci$, we can consider the points in which $\underline{A}$ and $\overline{A}$ are not differentiable, namely $\mathcal{D}^{\underline{A}}=\{x\in\R\tq \nexists \underline{A}^\prime(x)\}$ and $\mathcal{D}^{\overline{A}}=\{x\in\R\tq \nexists \overline{A}^\prime(x)\}$. In addition, we consider the respective image sets $\mathcal{L}^{\underline{A}}=\{\lim_{x\to x_0^+} \underline{A}(x),\lim_{x\to x_0^-}\underline{A}(x), \underline{A}(x_0)\tq x_0\in\mathcal{D}^{\underline{A}}\}$ and $\mathcal{L}^{\overline{A}}=\{\lim_{x\to x_0^+} \overline{A}(x),\lim_{x\to x_0^-}\overline{A}(x), \overline{A}(x_0)\tq x_0\in\mathcal{D}^{\overline{A}}\}$. Furhtermore, we can consider the mappings $\beta^{\underline{A},L},\beta^{\overline{A},L}:\R\to L^2$ defined by:
\begin{gather*}
    \beta^{\underline{A},L}(x)=\begin{cases}
        (0,0)&\text{ if }x\notin \underline{A}_0\\
        (1,1)&\text{ if }x\in \underline{A}_1\\
        (\max\limits_{x\in \underline{A}_\alpha}\{\alpha\in L \},\min\limits_{x\notin \underline{A}_\alpha}\{\alpha\in L\})&\text{ if }x\in \underline{A}_0\setminus \underline{A}_1\\
    \end{cases},\\
     \beta^{\overline{A},L}(x)=\begin{cases}
        (0,0)&\text{ if }x\notin \overline{A}_0\\
        (1,1)&\text{ if }x\in \overline{A}_1\\
        (\max\limits_{x\in \overline{A}_\alpha}\{\alpha\in L \},\min\limits_{x\notin \overline{A}_\alpha}\{\alpha\in L\})&\text{ if }x\in \overline{A}_0\setminus \overline{A}_1\\
    \end{cases},
\end{gather*}
with the notation $\beta^{\underline{A},L}(x)=(\beta_1^{\underline{A},L}(x),\beta_2^{\underline{A},L}(x)), \beta^{\overline{A},L}(x)=(\beta_1^{\overline{A},L}(x),\beta_2^{\overline{A},L}(x))\pt x\in\R$. On this basis, we may state the analogous result for DoC-IT2MFs of Theorem \ref{theo:charact}.

\begin{theorem}\label{theo:charcT2}
    Let $A=(\underline{A},\overline{A}): \mathbb{R} \to \ci$ be a DoC-IT2MF, and consider a finite set $L$ of $\alpha$-levels such that $\cl^{\underline{A}}, \cl^{\overline{A}}\subseteq L$. Then, the mappings $\underline{A}_L,\overline{A}_L: \mathbb{R} \to [0, 1]$ defined by
   \begin{gather*}
       \underline{A}_L(x) = \min \left\{ \frac{x - \underline{A}^+_1}{\underline{A}^+_2 - \underline{A}^+_1}, \frac{x - \underline{A}^-_1}{\underline{A}^-_2 - \underline{A}^-_1} \right\} (\beta^{\underline{A},L}_2(x) - \beta^{\underline{A},L}_1(x)) + \beta^{\underline{A},L}_1(x),\\
       \overline{A}_L(x) = \min \left\{ \frac{x - \overline{A}^+_1}{\overline{A}^+_2 - \overline{A}^+_1}, \frac{x - \overline{A}^-_1}{\overline{A}^-_2 - \overline{A}^-_1} \right\} (\beta^{\overline{A},L}_2(x) - \beta^{\overline{A},L}_1(x)) + \beta^{\overline{A},L}_1(x)
   \end{gather*}
    where $[\underline{A}^-_1, \underline{A}^+_1] = \underline{A}_{\beta^{\underline{A},L}_1(x)}$, $[\underline{A}^-_2, \underline{A}^+_2] = \underline{A}_{\beta^{\underline{A},L}_2(x)}$,   $[\overline{A}^-_1, \overline{A}^+_1] = \overline{A}_{\beta^{\overline{A},L}_1(x)}$, $[\overline{A}^-_2, \overline{A}^+_2] = \overline{A}_{\beta^{\overline{A},L}_2(x)}$, for all $x \in \mathbb{R}$, satisfy $(\underline{A}_L,\overline{A}_L)(x) = A(x)$ for all $x \in \mathbb{R}$.
\end{theorem}
\begin{proof}
    Under the described setup, the result is a direct consequence of Theorems \ref{theo:charact} and \ref{theo:Sumprodt2}
\end{proof}
With this characterization, it is possible to simplify the computations of additions and products by scalars of DoC-IT2MFs.

 \begin{cor}[Addition of DoC-IT2MFs]\label{cor:sum-T2}
        Let $A=(\underline{A},\overline{A}),B=(\underline{B},\overline{B}):\R\to\ci$ be two DoC-IT2MFs and consider a finite set $L$ of $\alpha$-levels such that $\mathcal{L}^{\underline{A}}, \mathcal{L}^{\overline{A}},\mathcal{L}^{\underline{B}},\mathcal{L}^{\overline{B}}\subseteq L$. Then, the mapping $(A+B)^L=(\underline{(A+B)^L},\overline{(A+B)^L}):\R\to\ci$ defined by  
        \begin{gather*}
            \underline{(A+B)^L}=(\underline{A}+\underline{B})^L,\\
            \overline{(A+B)^L}=(\overline{A}+\overline{B})^L,
        \end{gather*}
        satisfy   $(A+B)^L(x)=(A\oplus B)(x)\pt x\in\R$.
    \end{cor}
    \begin{proof}
        The result is a consequence of Theorem \ref{theo:charcT2} and Corollary \ref{cor:sum-pwlfn}.
    \end{proof}
      \begin{cor}[Product by scalars for DoC-IT2MFs]\label{cor:prod_T2}
        Let $A=(\underline{A},\overline{A}):\R\to\ci$ be a  DoC-MF and consider $r>0$ and a finite set $L$ of $\alpha$-levels such that $\mathcal{L}^{\underline{A}},\mathcal{L}^{\overline{A}}\subseteq L$. Then, the mapping $(rA)^L=(\underline{(rA)^L},\overline{(rA)^L}):\R\to\ci$ defined by
        \begin{gather*}            \underline{(rA)^L}=(r\underline{A})^L,\\
            \overline{(rA)^L}=(r\overline{A})^L,
        \end{gather*}
        satisfies $(rA)^L(x)=(r\odot A)(x)\pt x\in\R$.
    \end{cor}
    \begin{proof}
        This statement is consequence of Theorem \ref{theo:charcT2} and Corollary \ref{cor:prod-pwlfn}.
    \end{proof}
   \begin{cor}[Weighted Average DoC-IT2MFs]\label{corol:WA_T2}
       Let $A^1,...,A^n:\R\to\ci$ be $n\in\N$ DoC-IT2MFs such that $\sup(A^j)\subseteq\io\pt j=1,...,n$ and consider a finite set $L$ of $\alpha$-levels such that  $\mathcal{L}^{\underline{A}^{j}},\mathcal{L}^{\overline{A}^{j}}\subseteq L \pt j=1,...,n $. Additionally, let us consider a weight vector $w=(w_1,...,w_n)\in\io$ such that $\sum_{j=1}^nw_j=1$. Then, the mapping $\Phi_w(A^1,...,A^n)^L:\R\to\ci$ defined by
       \begin{gather*}
           \underline{\Phi_w(A^1,...,A^n)^L}=\Phi_w(\underline{A}^1,...,\underline{A}^n)^L,\\ \overline{\Phi_w(A^1,...,A^n)^L}=\Phi_w(\overline{A}^1,...,\overline{A}^n)^L,
       \end{gather*}
    
     satisfies $\Phi_w(A^1,...,A^n)^L(x)=\bigoplus\limits_{j=1}^n(w_j\odot A^j)(x)\pt x\in\R$ and $\sup(\Phi_w(A^1,...,A^n)^L)\subseteq\io$.
   \end{cor}
   \begin{proof}
       In this case, the result follows from Theorem \ref{theo:charcT2} and Corollary \ref{corol:WA}.
   \end{proof}

\begin{theorem}
    Let $\leqslant^\cf$ be an admissible order for fuzzy numbers. Then, the following binary relations are total orders on $\cff$:
    \begin{equation*}
        \begin{split}
            A\leqslant^\cff_1 B\iff \begin{cases}
            \underline{A}<^\cf \underline{B}\\
            \text{
            or
            }\\
            \underline{A}= \underline{B}\text{ and } \overline{A}\leqslant^\cf \overline{B}
        \end{cases}\\
        A\leqslant^\cff_2 B\iff \begin{cases}
            \overline{A}<^\cf \overline{B}\\
            \text{
            or
            }\\
            \overline{A}= \overline{B}\text{ and } \underline{A}\leqslant^\cf \underline{B}
        \end{cases}
        \end{split}
    \end{equation*}
    In addition, both of them are admissible in the sense that for any $A, B\in\cff$ such that $\underline{A}\leqslant^\cf_0 \underline{B}$ and $\overline{A}\leqslant^\cf_0 \overline{B}$, $ A\leqslant^\cff_1 B$ and $ A\leqslant^\cff_2 B$ hold.
\end{theorem}
\begin{proof}
    Let us consider the binary relation $\leqslant^\cff_1$. Clearly, $\leqslant^\cff_1$ is reflexive, i.e., $A \leqslant^\cff_1 B\pt A\in\cff$. Assume that $A,B\in\cff$ are such that $A\leqslant^\cff_1 B$ and $B \leqslant^\cff_1 A$ to analyze the following cases:
    (i) $\underline{A}<^\cf\underline{B}$ and $\underline{B}<^\cf\underline{A}$, (ii) $\underline{A}<^\cf\underline{B}$ and $\underline{A}=\underline{B}$, $\overline{B}\leqslant^\cf \overline{A}$, (iii) $\underline{A}=\underline{B}$,  $\overline{A}\leqslant^\cf \overline{B}$ and $\underline{B}<^\cf\underline{A}$, and (iv) $\underline{A}=\underline{B}$, $\overline{A}\leqslant^\cf \overline{B}$ and $\overline{B}\leqslant^\cf \overline{A}$. From these scenarios, only the last one may hold under $A\leqslant^\cff_1 B$ and $B \leqslant^\cff_1 A$. In such case, since  $\overline{A}\leqslant^\cf \overline{B}$ and $\overline{B}\leqslant^\cf \overline{A}$, the antisymmetry of $\leqslant^\cf$ implies $\overline{A}=\overline{B}$ and thus $A=B$, which proves the antisymmetry of $\leqslant^\cff_1$. Now, consider $A\leqslant^\cff_1 B$ and $B\leqslant^\cff_1 C$. Again, we have four cases:
    \begin{itemize}[label={--}]
        \item $\underline{A}<^\cf\underline{B}$ and $\underline{B}<^\cf\underline{C}$, which implies $\underline{A}<^\cf\underline{C}$,
        \item $\underline{A}<^\cf\underline{B}$ and $\underline{B}=\underline{C}$, $\overline{B}<^\cf\overline{C}$, which implies $\underline{A}<^\cf\underline{C}$,
        \item $\underline{A}=\underline{B}, \overline{A}\leqslant^\cf\overline{B}$ and $\underline{B}\leqslant^\cf\underline{C}$, which implies $\underline{A}\leqslant^\cf\underline{C}$,
        \item $\underline{A}=\underline{B}, \overline{A}\leqslant^\cf\overline{B}$, $\underline{B}=\underline{C}, \overline{B}\leqslant^\cf\overline{C}$, which implies $\underline{A}=\underline{C}$ and $\overline{A}\leqslant^\cf\overline{C}$.
    \end{itemize}
    In any case, $A\leqslant^\cff_1 C$, which is the transitivity of $\leqslant^\cff_1$. Obviously, $\leqslant^\cff_1$ is a total binary relation (because $\leqslant^\cf$ is total) and the admissibility is straight forward from the admissibility of $\leqslant^\cf.$
\end{proof}

\section{Conclusion}\label{sec: Conclusions}
\noindent In this study, we have presented a methodology for constructing IT2MFs model for linguistic terms associated with a given linguistic scale by adopting a co-constructive approach where decision analysts interact with DM to elicit his/her subjective judgments. In the first step, by integrating the idea of building membership function from the pairwise comparisons in ratio scale with deck-of-cards, our approach (DoC-T1MF) enhances the interpretability and reliability of fuzzy set-based linguistic models to deal with complexity and uncertainty inherent in the decision-making process. In the second step, the DoC-IT2MF methodology focuses on capturing interpersonal uncertainty (hesitation) of the DM through an intuitive mechanism of imprecise allocation of blank cards to build IT2-MF for a set of T1MF. This facilitates the easy and meaningful interpretation of the embedded T1MF and uncertainty on membership degree allocation. In the end, it allows us to build an effective and accurate personalized linguistic scale for the decision-maker.

Finally, a theoretical framework has been developed to manage such types of information in decision-making, particularly the MCDM framework based on additive utility theory by defining appropriate aggregation rules and ordering. This framework will allow us to apply the proposed methodology in solving real-world decision-making, especially involving experts’ subjective judgments to make more informed, acceptable and effective decisions. 

Despite its advantages, the proposed methodology presents several avenues for future research. First, extending the approach to higher-order fuzzy models, such as General Type-2 Fuzzy Sets (GT2FSs), could further enhance the handling of uncertainty and vagueness. Second, investigating alternative preference elicitation techniques that incorporate active learning mechanisms could refine the construction process and reduce the cognitive burden. Lastly, applying the proposed framework to diverse real-life decision contexts is essential to validate and refine the construction methodology including the effective questioning protocols.

In conclusion, the proposed socio-technical co-constructive approach represents a significant step toward improving the modelling of linguistic uncertainty in the expert-driven decision-making process. Further advancements in this direction will have the potential to improve decision support systems for expert-driven decision-making under the computing with word paradigm.

\section*{Acknowledgements}
\addcontentsline{toc}{section}{\numberline{}Acknowledgements}
\noindent José Rui Figueira acknowledges the support by national funds through FCT (Fundação para a Ciência e a Tecnologia), Portugal under the project UIDB/00097/2020.  Bapi Dutta acknowledges the support of the Spanish Ministry of Economy and Competitiveness through the Ramón y Cajal Research grant (RYC2023-045020-I), Spain.  


\vfill\newpage
\section*{Appendix}\label{sec:appendix}
\addcontentsline{toc}{section}{\numberline{}Appendix}
\begin{theorem}\label{th: num_to_Card}
    Consider an $n+1$-tuple $x\in[0,1]^{n+1}$ ($n\in\N, n>1$) satisfying $0=x_0<x_1<...<x_n=1$. Let $m\in\N$ be an integer such that $\operatorname{floor}(10^mx_{i-1})<\operatorname{floor}(10^mx_i)\pt i=1,...,n$. Then, it is possible to use the DoC method with $10^m$ cards to approximate the $n+1$-tuple $x$ with precision $10^{-m}$.
\end{theorem}
\begin{proof}
       Let us define a rational approximation of the $n+1$-tuple $x$ as $r_i=\frac{\operatorname{floor}(10^mx_i)}{10^m}$ $\pt i=0,...,n$. Subsequently, consider the number of cards $c_i$ to place between the consecutive values $r_{i-1}$ and $r_{i}$ for $i=1,...,n$ and denote the total number of cards as $N=\sum_{j=1}^nc_j$. Note that, if the DoC method can be applied, it is necessary that 
    \begin{equation*}
        r_i=\frac{1}{N}\sum_{j=1}^{i}c_j \iff  r_i{N}=\sum_{j=1}^{i}c_j.
    \end{equation*}
    Therefore, if we take $N=10^m$, the quantities $r_iN$ are integers for $i=0,...,n$ that satisfy $r_{i-1}N<r_iN\pt i=1,...,n$. In such case, all the values $c_i$ are integers greater than one and they can be computed recursively as
    \begin{equation*}
        \begin{split}
            c_1&=10^mr_1=\operatorname{floor}(10^mx_1)\\
            c_i&=10^mr_i-\sum_{j=1}^{i-1}c_j=\operatorname{floor}(10^mx_i)-\sum_{j=1}^{i-1}c_j, i>1.
        \end{split}
    \end{equation*}
\end{proof}

\vfill\newpage

\addcontentsline{toc}{section}{\numberline{}References}
\bibliographystyle{model2-names}
\bibliography{manuscript}







\end{document}